\documentclass[3p,twocolumn]{elsarticle}

\usepackage{hyperref}
\usepackage{amssymb}
\usepackage{graphicx}
\usepackage{amsmath}
\usepackage{algorithm}
\usepackage[noend]{algpseudocode}
\usepackage{xcolor}

\usepackage{lineno}

\newtheorem{prop}{Proposition}
\newtheorem{theorem}{Theorem}[section]
\newenvironment{proof}[1][Proof]{\begin{trivlist}
\item[\hskip \labelsep {\bfseries #1}]}{\end{trivlist}}
\newenvironment{definition}[1][Definition]{\begin{trivlist}
\item[\hskip \labelsep {\bfseries #1}]}{\end{trivlist}}
\DeclareMathOperator*{\argmin}{arg\,min}

\makeatletter
\def\BState{\State\hskip-\ALG@thistlm}
\makeatother

\begin{document}

\begin{frontmatter}



\title{Piece-wise quadratic approximations of arbitrary error functions for fast and robust machine learning}


\author[LeicMath]{A.N. Gorban\corref{cor1}}
\ead{ag153@le.ac.uk}
\author[LeicMath]{E.M. Mirkes}
\ead{em322@le.ac.uk}
\author[Curie]{A. Zinovyev}
\ead{Andrei.Zinovyev@curie.fr}

\address[LeicMath]{Department of Mathematics, University of Leicester, Leicester, LE1 7RH, UK}
 \address[Curie]{Institut Curie, PSL Research University, Mines Paris Tech, Inserm, U900, F-75005, Paris, France}

\cortext[cor1]{Corresponding author}

\begin{abstract}
Most of machine learning approaches have stemmed from the application of minimizing the mean squared distance principle, based on the
computationally efficient quadratic optimization methods. However, when faced with high-dimensional and noisy data, the quadratic error functionals
demonstrated many weaknesses including high sensitivity to contaminating factors and dimensionality curse. Therefore, a lot of recent applications
in machine learning exploited properties of non-quadratic error functionals based on $L_1$ norm or even sub-linear potentials corresponding to quasinorms $L_p$ ($0<p<1$).
The back side of these approaches is increase in computational cost for optimization. Till so far, no approaches
have been suggested to deal with {\it arbitrary} error functionals, in a flexible and computationally efficient framework.
In this paper, we develop a theory and basic universal data approximation algorithms ($k$-means, principal components, principal manifolds and graphs, regularized and sparse regression),
based on piece-wise quadratic error potentials of subquadratic growth (PQSQ potentials). We develop a new and universal framework to
minimize {\it arbitrary sub-quadratic error potentials} using an algorithm with guaranteed fast convergence
to the local or global error minimum. The theory of PQSQ potentials is based on the notion of the cone of minorant functions,
and represents a natural approximation formalism based on the application of min-plus algebra.
The approach can be applied in most of existing machine learning methods, including methods of data approximation and regularized and sparse regression, leading to the improvement in the computational cost/accuracy trade-off. We demonstrate that on synthetic and real-life datasets PQSQ-based machine learning methods achieve orders of magnitude faster computational performance than the corresponding state-of-the-art methods, having similar or better approximation accuracy.
\end{abstract}


\end{frontmatter}


\section{Introduction}
\label{S:1}

Modern machine learning and artificial intelligence methods are revolutionizing many fields of science today, such as medicine, biology, engineering, high-energy physics and sociology, where large amounts of data have been collected due to the emergence of new high-throughput computerized technologies. Historically and methodologically speaking, many machine learning algorithms have been based on minimizing the mean squared error potential, which can be explained by tractable properties of normal distribution and existence of computationally efficient methods for quadratic optimization. However, most of the real-life datasets are characterized by strong noise, long-tailed distributions, presence of contaminating factors, large dimensions. Using quadratic potentials can be drastically compromised by all these circumstances: therefore, a lot of practical and theoretical efforts have been made in order to exploit the properties of non-quadratic error potentials which can be more appropriate in certain contexts. For example, methods of regularized and sparse regression such as lasso and elastic net based on the properties of $L_1$ metrics \cite{Tibshirani1996, Zou2005} found numerous applications in bioinformatics \cite{Barillot2012}, and $L_1$ norm-based methods of dimension reduction are of great use in automated image analysis \cite{Wright2010}. Not surprisingly, these approaches come with drastically increased computational cost, for example, connected with applying linear programming optimization techniques which are substantially more expensive compared to mean squared error-based methods.

In practical applications of machine learning, it would be very attractive to be able to deal with {\it arbitrary error potentials}, including those based on $L_1$ or fractional quasinorms $L_p$ ($0<p<1$), in a computationally efficient and scalable way. There is a need in developing methods allowing to tune the {\it computatio\-nal cost/accuracy of optimization} trade-off accordingly to various contexts.

In this paper, we suggest such a universal framework able to deal with a large family of error potentials. We exploit the fact that finding a minimum of a piece-wise quadratic function, or, in other words, a function which is the {\it minorant of a set of quadratic functionals}, can be almost as computationally efficient as optimizing the standard quadra\-tic potential. Therefore, if a given arbitrary potential (such as  $L_1$-based or fractional quasinorm-based) can be approximated by a piece-wise quadratic function, this should lead to relatively efficient and simple optimization algorithms.  It appears that only potentials of quadratic or subquadratic growth are possible in this approach: however, these are the most usefull ones in data analysis. We introduce a rich family of piece-wise quadratic potentials of subquadratic growth (PQSQ-potentials), suggest general approach for their optimization and prove convergence of a simple iterative algorithm in the most general case. We focus on the most used methods of data dimension reduction and regularized regression: however, potential applications of the approach can be much wider.

Data dimension reduction by constructing explicit low-dimensional approximators of a finite set of vectors is one of the most fundamental approach in data analysis. Starting from the classical data approximators such as $k$-means \cite{Lloyd1957} and linear principal components (PCA) \cite{Pearson1901On}, multiple generalizations have been suggested in the last decades (self-organizing maps, principal curves, principal manifolds, principal graphs, principal trees, etc.)\cite{Gorban2009,Gorban2008Principal} in order to make the data approximators more flexible and suitable for complex data structures.

We solve the problem of approximating a finite set of vectors ${\vec{x}_i}\in R^m,\, i=1,\ldots, N$ (data set) by a simpler object $L$ embedded into the data space, such that for each point $\vec{x}_i$ an approximation error $err(\vec{x}_i,L)$ function can be defined. We assume this function in the form

\begin{equation}\label{distance_function}
err(\vec{x}_i,L) = \min_{y\in L} \sum_k u(x_i^k-y^k),
\end{equation}
where the upper $k=1,\ldots ,m$ stands for the coordinate index, and $u(x)$ is a monotonously growing symmetrical function, which we will be calling the error potential. By data approximation we mean that the embedment of $L$ in the data space minimizes the error

$$
\sum_i err(\vec{x}_i,L) \rightarrow \min.
$$

Note that our definition of error function is coor\-dinate-wise (it is a sum of error potential over all coordinates).

The simplest form of the error potential is qua\-dratic $u(x)=x^2$, which leads to the most known data approximators: mean point ($L$ is a point), principal points ($L$ is a set of points) \cite{Flury1990}, principal components ($L$ is a line or a hyperplane) \cite{Pearson1901On}. In more advanced cases, $L$ can posses some regular properties leading to principal curves ($L$ is a smooth line or spline) \cite{Hastie1984}, principal manifolds ($L$ is a smooth low-dimensional surface) and principal graphs (eg., $L$ is a pluri-harmonic graph embedment) \cite{gorban2007topological,Gorban2009}.

There exist multiple advantages of using qua\-dra\-tic potential $u(x)$, because it leads to the most computationally efficient algorithms usually based on the splitting schema, a variant of expectation-mi\-ni\-mi\-za\-tion approach \cite{Gorban2009}. For example, $k$-means algorithm solves the problem of finding the set of principal points and the standard iterative Singular Value Decomposition finds principal components. However, qua\-dra\-tic potential is known to be sensitive to outliers in the data set. Also, purely qua\-dra\-tic potentials can suffer from the curse of dimensionality, not being able to robustly discriminate `close' and `distant' point neighbours in a high-dimen\-sio\-nal space \cite{Aggarwal2001}.

There exist several widely used ideas for increasing approximator's robustness in presence of strong noise in data such as: (1) using medians instead of mean values, (2) substituting quadratic norm by $L_1$ norm (e.g. \cite{Ding2006, hauberg2014}), (3) outliers exclusion or fixed weighting or iterative reweighting during optimizing the data approximators (e.g. \cite{Xu1995,Allende2004,kohonen2001self}), and (4) regularizing the PCA vectors by $L_1$ norm \cite{Jolliffe2003,Candes2011,Zou2006}. In some works, it was suggested to utilize `trimming' averages, e.g. in the context of the $k$-means clustering or some generalizations of PCA \cite{cuesta1997,hauberg2014}). In the context of regression, iterative reweighting is exploited to mimic the properties of $L_1$ norm \cite{Lu2015}. Several algorithms for constructing PCA with $L_1$ norm have been suggested \cite{Ke2005,Kwak2008,Brooks2013} and systematically benchmarked \cite{brooks2012pcal1,Park2014}. Some authors go even beyond linear metrics and suggests that fractional quasinorms $L_p$ ($0<p<1$) can be more appropriate in high-dimensional data approximation \cite{Aggarwal2001}.

However, most of the suggested approaches exploiting properties of non-quadratic metrics either represent useful but still arbitrary heuristics or are not sufficiently scalable. The standard approach for minimizing $L_1$-based norm consists in solving a linear programming task. Despite existence of many efficient linear programming optimizer implementations, by their nature these computations are much slower than the iterative methods used in the standard SVD algorithm or $k$-means.

In this paper, we provide implementations of the standard data approximators (mean point, $k$-means, principal components) using a PQSQ potential. As an other application of PQSQ-based framework in machine learning, we develop PQSQ-based regularized and sparse regression (imitating the properties of lasso and elastic net).

\section{Piecewise quadratic potential of subqua\-dratic growth (PQSQ)}
\label{S:2}

\subsection{Definition of the PQSQ potential}

Let us split all non-negative numbers $x\in R_{\geq 0}$ into $p+1$ non-intersecting intervals $R_0=[0;r_1), R_1=[r_1;r_2), \ldots , R_k=[r_k;r_{k+1}), \ldots , R_p=[r_p;\infty)$,  for a set of thresholds $r_1<r_2<\ldots <r_p$. For convenience, let us denote $r_0=0, r_{p+1} = \infty$. Piecewise quadratic potential is a continuous monotonously growing function $u(x)$ constructed from pieces of centered at zero parabolas $y=b_k+a_kx^2$, defined on intervals $x\in[r_k,r_{k+1})$, satisfying $y(r_i)=f(r_i)$  (see Figure~\ref{potential}):

\begin{equation}\label{PQSQ_f}
u(x)=
b_k+a_kx^2, \mbox{if } r_k \leq |x|<r_{k+1},\,  k=0, \ldots , p,
\end{equation}

\begin{equation}\label{PQSQ_acoeffs}
a_k = \frac{f(r_k)-f(r_{k+1})}{r_k^2-r_{k+1}^2},
\end{equation}

\begin{equation}\label{PQSQ_bcoeffs}
b_k = \frac{f(r_{k+1})r_k^2-f(r_{k})r_{k+1}^2}{r_k^2-r_{k+1}^2},
\end{equation}

\noindent where $f(x)$ is a majorating function, which is to be approximated (imitated) by $u(x)$. For example, in the simplest case $f(x)$ can be a linear function: $f(x)=x$, in this case, $\sum_k u(x^k)$ will approximate the $L_1$-based error function.

Note that accordingly to (\ref{PQSQ_acoeffs},\ref{PQSQ_bcoeffs}), $b_0=0, a_p=0, b_p=f(r_p)$. Therefore, the choice of $r_p$ can naturally create a `trimmed' version of error potential $u(x)$ such that some data points (outliers) do not have any contribution to the gradient of $u(x)$, hence, will not affect the optimization procedure. However, this set of points can change during minimizaton of the potential.

The condition of subquadratic growth consists in the requirement $a_{k+1}\leq a_{k}$ and $b_{k+1} \geq b_{k}$. To guarantee this, the following simple condition on $f(x)$ should be satisfied:

\begin{equation}
\label{eq:condition_function}
f'>0, \>\>\> f''x \leq f'.
\end{equation}
Therefore,  $f(x)$ is a monotonic  concave function of $q=x^2$: $$\frac{d^2f(\sqrt{q})}{dq^2}=\frac{1}{4x^2}f''(x)-\frac{1}{4x^3}f'(x)\leq 0.$$
In particular, $f(x)$ should grow not faster than any parabola $ax^2+c,\, c>0$, which is  tangent to $f(x)$.

\begin{figure}[ht]
\centering\includegraphics[width=0.75\linewidth]{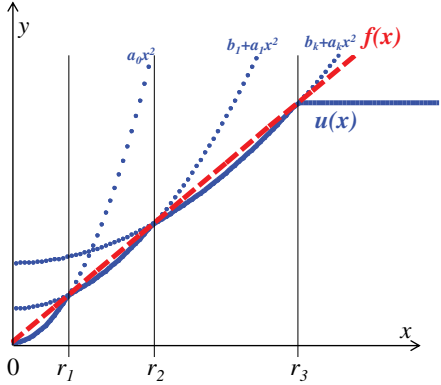}
\caption{Trimmed piecewise quadratic potential of subquadratic growth $u(x)$ (solid blue line) defined for the majorating function $f(x)$ (red dashed line) and several thresholds $r_k$. Dotted lines show the parabolas which fragments are used to construct $u(x)$. The last parabola is flat ($a_p=0$) which corresponds to trimmed potential. \label{potential}}
\end{figure}

\subsection{Basic approach for optimization}

In order to use the PQSQ potential in an algorithm, a set of $p$ interval thresholds $r_s^k, \, s=1, \ldots ,p$ for each coordinate $k=1, \ldots ,m$ should be provided. Matrices of $a$ and $b$ coefficients defined by (\ref{PQSQ_acoeffs},\ref{PQSQ_bcoeffs}) based on interval definitions: $a_s^k, b_s^k$, $s=0, \ldots ,p$, $k=1, \ldots ,m$ are computed separately for each coordinate $k$.

Minimization of PQSQ-based functional consists in several basic steps which can be combined in an algorithm:

1) For each coordinate $k$, split all data point indices into non-overlapping sets $\mathcal{R}_s^k$:

\begin{equation}
\mathcal{R}_s^k= \{i: r_{s}^k \leq |x_i^k-\beta^k_i| < r_{s+1}^k\},\,  s = 0, \ldots ,p,
\end{equation}

\noindent where $\bf{\beta}$ is a matrix which depends on the nature of the algorithm.

2) Minimize PQSQ-based functional where each set of points $\{x_{i\in \mathcal{R}_s^k}\}$ contributes
to the functional quadratically with coefficient $a_s^k$. This is a quadratic optimization task.

3) Repeat (1)-(2) till convergence.

\section{General theory of the piece-wise convex potentials as the cone of minorant functions}\label{ConvergenceSection}

In order to deal in most general terms with the data approximation algorithms based on PQSQ potentials, let us consider a general case where a potential can be constructed from a set of functions $\{q_i(x)\}$ with only two requirements: 1) that each $q_i(x)$ has a (local) minimum; 2) that the whole set of all possible $q_i(x)$s forms a cone. In this case, instead of the operational definition (\ref{PQSQ_f}) it is convenient to define the potential $u(x)$ as the minorant function for a set of functions as follows. For convenience, in this section, $x$ will notify a vector $\vec{x}\in R^m$.

Let us consider {\it a generating cone of functions} $Q$. We remind that the definition of a cone implies that for any $q(x)\in Q, p(x)\in Q$, we have $\alpha q(x)+\beta p(x)\in Q$, where $\alpha\geq 0, \beta \geq 0$.

For any finite set of functions $${q_1(x)\in Q, q_2(x)\in Q,\ldots , q_s(x)\in Q},$$ we define the minorant function (Figure \ref{OptimizationFigure}):
\begin{equation}\label{minorant}
u_{q_1,q_2,\ldots ,q_s}(x) = \min(q_1(x),q_2(x),\ldots ,q_s(x)).
\end{equation}

It is convinient to introduce a multiindex $$I_{q_1,q_2,\ldots ,q_s}(x)$$ indicating which particular function(s) $q_i$ corresponds to the value of $u(x)$, i.e.
\begin{equation}\label{minorant_index}
I_{q_1,q_2,\ldots,q_s}(x) = \{i|u_{q_1,q_2,\ldots ,q_s}(x)=q_i(x)\}.
\end{equation}

For a cone $Q$ let us define a set of all possible minorant functions $\mathbb{M}(Q)$

\begin{equation}\label{minorant_index1}
\begin{split}
\mathbb{M}(Q) = \{ u_{q_{i_1},q_{i_2},\ldots ,q_{i_n}} | q_{i_1}\in Q, q_{i_2}\in Q, \\ q_{i_n}\in Q,\, n = 1,2,3,\dots \}.
\end{split}
\end{equation}

\begin{prop}\label{Misacone}
$\mathbb{M}(Q)$ is a cone.
\end{prop}
\begin{proof}
For any two minorant functions $$u_{q_{i_1},q_{i_2},\ldots ,q_{i_k}}, u_{q_{j_1},q_{j_2},\ldots ,q_{j_s}}\in \mathbb{M}(Q)$$ we have
\begin{equation}
\begin{split}
\alpha u_{q_{i_1},q_{i_2},\ldots ,q_{i_k}} + \beta u_{q_{j_1},q_{j_2},\ldots ,q_{j_s}} = \\ u_{\{\alpha q_{i_p}+\beta q_{j_r}\}} \in \mathbb{M}(Q), \\ p=1,\dots,k,\, r=1,\dots,s,
\end{split}
\end{equation}

\noindent where ${\{\alpha q_{i_p}+\beta q_{j_r}\}}$ is a set of all possible linear combinations of functions from $\{q_{i_1},q_{i_2},\ldots ,q_{i_k}\}$ and $\{q_{j_1},q_{j_2},\ldots ,q_{j_s}\}$.
\end{proof}

\begin{prop}\label{Mrestrictedisacone}
Any restriction of $\mathbb{M}(Q)$ onto a linear manifold $L$ is a cone.
\end{prop}
\begin{proof}
Let us denote $q(x)|_L$ a restriction of $q(x)$ function onto $L$, i.e. $q(x)|_L = \{q(x)|x\in L\}$. $q(x)|_L$ is a part of $Q$.
Set of all $q(x)|_L$ forms a restriction $Q|_L$ of $Q$ onto $L$. $Q|_L$ is a cone, hence, $\mathbb{M}(Q)|_L = \mathbb{M}(Q|_L)$ is a cone (Proposition \ref{Misacone}).
\end{proof}

\begin{definition}{\it Splitting algorithm} minimizing $$u_{q_{1},q_{2},\ldots ,q_{n}}(x)$$ is defined as {Algorithm \ref{MinorantMinimum}}.
\end{definition}

\begin{algorithm}
\caption{Finding local minimum of a minorant function $u_{q_{1},q_{2},\ldots ,q_{n}}(x)$}\label{MinorantMinimum}
\begin{algorithmic}[1]
\Procedure{Minimizing minorant function}{}
\State $\textit{initialize } x \gets x_0$
\BState \emph{repeat until stopping criterion has been met}:
\State $\textit{compute multiindex } I_{q_1,q_2,\ldots ,q_s}(x)$
\State \textbf{for all } $i\in I_{q_1,q_2,\ldots ,q_s}(x)$
\State $x_i = \argmin\,q_i(x)$
\State \textbf{end for}
\State $\textit{select optimal } x_i:$
\State $x_{opt} \gets \argmin_{x_i} u(x_i)$
\State $x \gets x_{opt}$
\State $\textit{stopping criterion:}$ check if the multiindex
$I_{q_1,q_2,\ldots ,q_s}(x)$ does not change compared
to the previous iterationf
\EndProcedure
\end{algorithmic}
\end{algorithm}

\begin{theorem}\label{theoremConvergence} Splitting algorithm ({Algorithm \ref{MinorantMinimum}}) for minimizing $u_{q_{1},q_{2},\ldots ,q_{n}}(x)$ converges in a finite number of steps.
\end{theorem}
\begin{proof}
Since the set of functions $\{q_{1},q_{2},\ldots ,q_{n}\}$ is finite then we only have to show that at each step the value of the function $u_{q_{1},q_{2},\ldots ,q_{n}}(x)$ can not increase. For any $x$ and the value $x' = \argmin q_i(x)$ for $i\in I_{q_1,q_2,\ldots ,q_s}(x)$ we can have only two cases:

(1) Either $I_{q_1,q_2,\ldots ,q_s}(x)=I_{q_1,q_2,\ldots ,q_s}(x')$ (convergence, and in this case $q_{i'}(x')=q_i(x')$ for any $i'\in I_{q_1,q_2,\ldots ,q_s}(x')$);

(2) Or $u_{q_{1},q_{2},\ldots ,q_{n}}(x')<u_{q_{1},q_{2},\ldots ,q_{n}}(x)$ since, accordingly to the definition (\ref{minorant}), $q_{i'}(x')<q_i(x)$, for any $i'\in I_{q_1,q_2,\ldots ,q_s}(x'), i\in I_{q_1,q_2,\ldots ,q_s}(x)$ (see Figure \ref{OptimizationFigure}).
\end{proof}

Note that in {Algorithm \ref{MinorantMinimum}} we do not specify exactly the way to find the local minimum of $q_i(x)$. To be practical, the cone $Q$ should contain only functions for which finding a local minimum is fast and explicit. Evident candidates for this role are positively defined quadratic functionals $q(x)=q_{0}+(\vec{q_1},x)+(x,\mathbb{Q}_2x)$, where $\mathbb{Q}_2$ is a positively defined symmetric matrix. Any minorant function (\ref{minorant}) constructed from positively defined quadratic functions will automatically provide subquadratic growth, since the minorant can not grow faster than any of the quadratic forms by which it is defined.

Operational definition of PQSQ given above (\ref{PQSQ_f}), corresponds to a particular form of the quadratic functional, with $\mathbb{Q}_2$ being diagonal matrix. This choice corresponds to the coordinate-wise definition of data approximation error function (\ref{distance_function}) which is particularly simple to minimize. This circumstance is used in {Algorithms \ref{PQSQ_Mean_Algorithm},\ref{PQSQ_PC1}}.

\begin{figure}[ht]
\centering\includegraphics[width=0.75\linewidth]{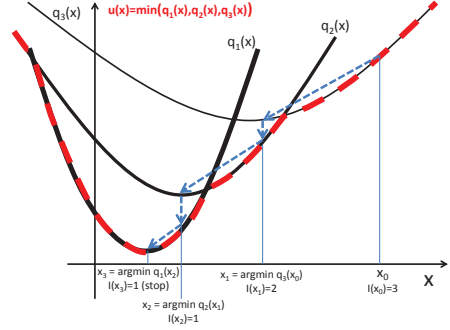}
\caption{Optimization of a one-dimensional minorant function $u(x)$, defined by three functions $q_1(x),q_2(x),q_3(x)$ each of which has a local     minimum.
Each optimization step consists in determining which $q_{I(x)}(x)=u(x)$ and making a step into the local minimum of $q_{I(x)}$. \label{OptimizationFigure}}
\end{figure}

\section{Commonly used data approximators with PQSQ potential}

\subsection{Mean value and $k$-means clustering in PQSQ approximation measure}

Mean vector $\bar{X}_L$  for a set of vectors $X=\{x_i^k\}$, $i=1,\ldots ,N, \, k=1,\ldots ,m$ and an approximation error defined by potential $f(x)$ can be defined as a point minimizing the mean error potential for all points in  $X$:

\begin{equation}\label{meanDef}
\sum_i\sum_k f(x_i^k-\bar{X}^k) \rightarrow \min.
\end{equation}

For Euclidean metrics $L_2$ ($f(x)=x^2$) it is the usual arithmetric mean.

For $L_1$ metrics ($f(x)=|x|$), (\ref{meanDef}) leads to the implicit equation $\#(x_i^k>\bar{X}^k)=\#(x_i^k<\bar{X}^k)$, where $\#$ stands for the number of points, which corresponds to the definition of median. This equation can have a non-unique solution in case of even number of points or when some data point coordinates coincide: therefore, definition of median is usually accompanied by heuristics used for breaking ties, i.e. to deal with non-uniquely defined rankings. This situation reflects the general situation of existence of multiple local minimuma and possible non-uniqueness of global minimum of (\ref{meanDef}) (Figure~\ref{MeanValueFigure}).

For PQSQ approximation measure (\ref{PQSQ_f}) it is difficult to write down an explicit formula for computing the mean value
corresponding to the global minimum of (\ref{meanDef}).
In order to find a point $\bar{X}_{PQSQ}$ minimizing mean $PQSQ$ potential, a simple iterative algorithm can be used (Algorithm~\ref{PQSQ_Mean_Algorithm}). The suggested algorithm converges to the local minimum which depends on the initial point approximation.

\begin{algorithm}
\caption{Computing PQSQ mean value}\label{PQSQ_Mean_Algorithm}
\begin{algorithmic}[1]
\Procedure{PQSQ Mean Value}{}
\State $\textit{define intervals } r_s^k, s=0, \ldots ,p, k=1, \ldots ,m$
\State $\textit{compute coefficients } a_s^k$
\State $\textit{initialize } \bar{X}_{PQSQ}  \linebreak \textit{eg., by arithmetic mean}$
\BState \emph{repeat till convergence of $\bar{X}_{PQSQ}$}:
\State \textbf{for each } \textit{coordinate } $k$
\State \textit{define sets of indices}
\begin{equation}
\begin{split}
\mathcal{R}_s^k=\{i: r_{s}^k \leq |x_i^k-\bar{X}_{PQSQ}^k| < r_{s+1}^k\}, \\ s = 0,\dots,p \nonumber
\end{split}
\end{equation}
\State \textit{compute new approximation for } $\bar{X}_{PQSQ}$:
\State $\bar{X}_{PQSQ}^k \gets \frac{\sum_{s=1, \ldots ,p}a_s^k\sum_{i\in \mathcal{R}_s^k}x_i^k}{\sum_{s=1, \ldots ,p}a_s^k|\mathcal{R}_s^k|}$
\State \textbf{end for}
\State \textbf{goto} \emph{repeat till convergence}
\EndProcedure
\end{algorithmic}
\end{algorithm}

Based on the PQSQ approximation measure and the algorithm for computing the PQSQ mean value ({Algorithm \ref{PQSQ_Mean_Algorithm}}), one can construct the PQSQ-based $k$-means clustering procedure in the usual way, splitting estimation of cluster centroids given partitioning of the data points into $k$ disjoint groups, and then re-calculating the partitioning using the PQSQ-based proximity measure.

\begin{figure*}[ht]
\centering\includegraphics[width=0.9\linewidth]{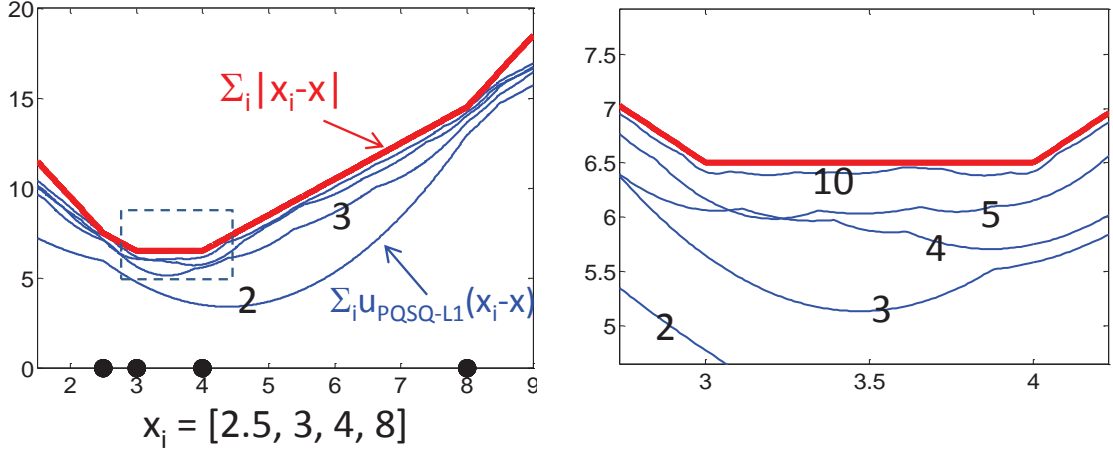}
\caption{Minimizing the error to a point (finding the mean value) for a set of 4 points (shown by black circles). Solid red line corresponds to $L_1$-based error. Thing blue lines correspond to PQSQ error potential imitating the $L_1$-based error. Several choices of PQSQ potential for different numbers of intervals (indicated by a number put on top of the line) is illustrated. On the right panel a zoom of a particular region of the left plot is shown. Neither function ($L_1$-based or PQSQ-based) possesses a unique local minimum. Moreover, $L_1$-based error function has infinite number of points corresponding to the global minimum (any number between 3 and 4), while PQSQ error function has several local minimuma in [3;4] interval which exact positions are sensitive to the concrete choice of PQSQ parameters (interval definitions).}\label{MeanValueFigure}
\end{figure*}

\subsection{Principal Component Analysis (PCA) in PQSQ metrics}

Accordingly to the classical definition of the first principal component, it is a line best fit to the data set $X$ \cite{Pearson1901On}. Let us define a line in the parametric form $\vec{y}=\vec{V}\nu+\vec{\delta}$, where $\nu \in R^1$ is the parameter. Then the first principal component will be defined by vectors $\vec{V}, \vec{\delta}$ satisfying

\begin{equation}
\sum_i\sum_k u(x_i^k-V^k\nu_i-\delta^k) \rightarrow \min,
\end{equation}

\noindent where

\begin{equation}
\nu_i = \arg \min_s \sum_k u(x_i^k-V^ks-\delta^k).
\end{equation}

The standard first principal component (PC1) corresponds to $u(x)=x^2$ when the vectors $\vec{V}, \vec{\delta}$ can be found by a simple iterative splitting algorithm for Singular Value Decomposition (SVD). If $X$ does not contain missing values then $\vec{\delta}$ is the vector of arithmetic mean values. By contrast, computing $L_1$-based principal components ($u(x)=|x|$) represents a much more challenging optimization problem \cite{Brooks2013}. Several approximative algorithms for computing $L_1$-norm PCA have been recently suggested and benchmarked \cite{Ke2005,Kwak2008,Brooks2013,brooks2012pcal1,Park2014}. To our knowledge, there have not been a general efficient algorithm suggested for computing PCA in case of arbitrary approximation measure for some monotonous function $u(x)$.

Computing PCA based on PQSQ approximation error is only slightly more complicated than computing the standard $L_2$ PCA by SVD. Here we provide a pseudo-code ({Algorithm \ref{PQSQ_PC1}}) of a simple iterative algorithm (similar to {Algorithm \ref{PQSQ_Mean_Algorithm}}) with guaranteed convergence (see Section \ref{ConvergenceSection}).

\begin{algorithm}
\caption{Computing PQSQ PCA}\label{PQSQ_PC1}
\begin{algorithmic}[1]
\Procedure{PQSQ First Principal Component}{}
\State $\textit{define intervals } r_s^k, s=0, \ldots ,p, \, k=1, \ldots ,m$
\State $\textit{compute coefficients } a_s^k$
\State $\vec{\delta} \gets \bar{X}_{PQSQ}$
\State $\textit{initialize } \vec{V} \textit{ : eg., by $L_2$-based PC1}$
\State $\textit{initialize } \{\nu_i\} \textit{ : eg., by } \linebreak $$\nu_i = \frac{\sum_k V^k(x_i^k-\delta^k)}{\sum_k (V^k)^2}$$ $
\BState \emph{repeat till convergence of $\vec{V}$}:
\State $\textit{normalize } \vec{V} \textit{ : } \vec{V} \gets \frac{\vec{V}}{\|\vec{V}\|}$
\State \textbf{for each } \textit{coordinate } $k$
\State \textit{define sets of indices}
\begin{equation}
\begin{split}
\mathcal{R}_s^k=\{i: r_{s}^k \leq |x_i^k-V^k\nu_i-\delta^k| < r_{s+1}^k\}, \\ s = 0, \ldots ,p \nonumber
\end{split}
\end{equation}
\State \textbf{end for}
\State \textbf{for each } \textit{data point } $i$ and \textit{coordinate } $k$
\State \textit{find all $s_{i,k}$ such that $i\in \mathcal{R}_{s_{i,k}}^k$}
\State \textbf{if}\textit{ all $a^k_{s_{i,k}}=0$ \textbf{then} $\nu'_i \gets 0$} \textbf{else}
\State
$$\nu'_i \gets \frac{\sum_{k}a_{s_{i,k}}^kV^k(x_i^k-\delta^k)}{\sum_{k}a_{s_{i,k}}^k(V^k)^2}$$
\State \textbf{end for}
\State \textbf{for each } \textit{coordinate } $k$
$$
V^k \gets \frac{\sum_s a_s^k \sum_{i\in \mathcal{R}_s^k}(x_i^k-\delta^k)\nu_i}{\sum_s a_s^k \sum_{i\in \mathcal{R}_s^k}(\nu_i)^2}
$$
\State \textbf{end for}
\State \textbf{for each } $i$ :
\State $\nu_i \gets \nu'_i$
\State \textbf{end for}
\State \textbf{goto} \emph{repeat till convergence}
\EndProcedure
\end{algorithmic}
\end{algorithm}

Computation of second and further principal components follows the standard deflation approach: projections of data points onto the previously computed component are subtracted from the data set, and the algorithm is applied to the residues. However, as it is the case in any non-quadratic metrics, the resulting components can be non-orthogonal or even not invariant with respect to the dataset rotation. Moreover, unlike $L_2$-based principal components, the {Algorithm \ref{PQSQ_PC1}} does not always converge to a unique global minimum; the computed components can depend on the initial estimate of $\vec{V}$. The situation is somewhat similar to the standard $k$-means algorithm. Therefore, in order to achieve the least possible approximation error to the linear subspace, $\vec{V}$ can be initialized randomly or by data vectors $\vec{x}_i$ many times and the deepest in PQSQ approximation error (\ref{distance_function}) minimum should be selected.

How does the {Algorithm \ref{MinorantMinimum}} serve a more abstract version of the {Algorithms \ref{PQSQ_Mean_Algorithm},\ref{PQSQ_PC1}}? For example, the `variance' function $m(\vec{x})=\frac{1}{N}\sum_j u(\vec{x}_j-\vec{x})$ to be minimized in {Algorithm \ref{PQSQ_Mean_Algorithm}} uses the generating functions in the form $Q = \{b_{ji}^k+\sum_k a_{ji}^k(x^k-x_j^k)^2\}$, where $i$ is the index of the interval in (\ref{PQSQ_f}). Hence, $m(x)$ is a minorant function, belonging to the cone $\mathbb{M}(Q)$, and must converge (to a local minimum) in a finite number of steps accordingly to Theorem \ref{theoremConvergence}.

\subsection{Nonlinear methods: PQSQ-based Principal \\ Graphs and Manifolds}

In a series of works, the authors of this article introduced a family of methods
for constructing principal objects
based on graph approximations (e.g., principal curves, principal manifolds, principal trees),
which allows constructing explicit non-linear data approximators
(and, more generally, approximators with non-trivial topologies, suitable for approximating,
e.g., datasets with branching or circular topology) \cite{Gorban1999, Gorban2001ihespreprint, gorban2001method, gorban2005elastic, gorban2007topological,Gorban2008Principal,Gorban2009,Gorban2010}. The methodology is
based on optimizing a piece-wise quadratic {\it elastic energy} functional (see short description below).
A convenient graphical user interface was developed with
implementation of some of these methods \cite{Gorban2014}.

Let $G$ be a simple undirected graph with set of vertices $Y$ and
set of edges $E$. For $k \geq 2$ a $k$-star in $G$ is a subgraph
with $k+1$ vertices $y_{0,1, \ldots , k} \in Y$ and $k$ edges $\{(y_0,
y_i) \ | \ i=1,\ldots , k\} \subset E$. Suppose for each $k\geq 2$, a
family $S_k$ of $k$-stars in $G$ has been selected. We call a graph
$G$ with selected families of $k$-stars $S_k$ an {\it elastic graph}
if, for all $E^{(i)} \in E $ and $S^{(j)}_k \in S_k$, the
correspondent elasticity moduli $\lambda_i > 0$ and $\mu_{kj}
> 0$ are defined. Let  $E^{(i)}(0),E^{(i)}(1)$ be vertices of an
edge $E^{(i)}$ and $S^{(j)}_k (0),\ldots , S^{(j)}_k (k)$ be vertices
of a $k$-star  $S^{(j)}_k $ (among them, $S^{(j)}_k (0)$ is the
central vertex).

For any map $\phi:Y \to R^m$ the {\it energy of the
graph} is defined as

\begin{equation}\label{elastic_energy}
\begin{split}
U^{\phi}{(G)}:=  \sum_{E^{(i)}} \lambda_i
\left\|\phi(E^{(i)}(0))-\phi(E^{(i)}(1)) \right\| ^2 + \\
+ \sum_{S^{(j)}_k}\mu_{kj} \left\|\sum _ {i=1}^k \phi(S^{(j)}_k
(i))-k\phi(S^{(j)}_k (0)) \right\|^2. \nonumber
\end{split}
\end{equation}

For a given map $\phi: Y \to R^m$ we divide the dataset $D$ into
node neighborhoods $K^y, \, y\in Y$. The set $K^y$ contains the data points for
which the node $\phi(y)$ is the closest one in $\phi(y)$. The {\it
energy of approximation} is:

\begin{equation}\label{approximation_term}
U^{\phi}_A(G,D)= \sum_{y \in Y} \sum_{ x \in K^y} w(x) \|x-
\phi(y)\|^2,
\end{equation}
where $w(x) \geq 0$ are the point weights. Simple and fast algorithm for minimization of the energy

\begin{equation}\label{globalStandardEnergy}
U^{\phi}=U^{\phi}_A(G,D)+U^{\phi}{(G)}
\end{equation}

\noindent is the splitting algorithm, in the spirit of the classical $k$-means clustering: for a given
system of sets $\{K^y \ | \ y \in Y \}$ we minimize $U^{\phi}$ (optimization step, it
is the minimization of a positive quadratic functional), then for a
given $\phi$ we find new $\{K^y\}$ (re-partitioning), and so on; stop when no change.

Application of PQSQ-based potential is straightforward in this approach. It consists
in replacing (\ref{approximation_term}) with

\begin{equation}\label{approximation_term_PQSQ}
U^{\phi}_A(G,D)= \sum_{y \in Y} \sum_{ x \in K^y} w(x) \sum_k u(x^k-
\phi(y^k)),\nonumber
\end{equation}

\noindent where $u$ is a chosen PQSQ-based error potential. Partitioning of the dataset
into $\{K^y\}$ can be also based on calculating the minimum PQSQ-based error to $y$, or can
continue enjoying nice properties of $L_2$-based distance calculation.

\section{PQSQ-based regularized regression}

\subsection{Regularizing linear regression with PQSQ potential}

One of the major application of non-Euclidean norm properties in machine learning is using non-quadratic terms for penalizing large absolute values of regression coefficients \cite{Tibshirani1996, Zou2005}. Depending on the chosen penalization term, it is possible to achieve various effects such as sparsity or grouping coefficients for redundant variables. In a general form, regularized regression solves the following optimization problem

\begin{equation}\label{RegularizedRegression}
\frac{1}{N}\sum_{i=1}^N \left(y_i-\sum_{k=1}^m\beta^k x_i^k\right)^2+\lambda f(\vec{\beta}) \rightarrow \min ,
\end{equation}
 where $N$ is the number of observations, $m$ is the number of independent variables in the matrix $\{x_i^k\}$, $\{y_i\}$ are dependent variables (to be predicted), $\lambda$ is an internal parameter controlling the strength of regularization (penalty on the amplitude of regression coefficients $\beta$), and $f(\vec{z})$ is the regularizer function, which is $f(\vec{z})=\|\vec{z}\|_{L2}^2$ for ridge regression, $f(\vec{z})= \|\vec{z}\|_{L1}$ for lasso and $f(\vec{z})=\frac{1-\alpha}{\alpha}\|\vec{z}\|_{L2}^2+\alpha\|\vec{z}\|_{L1}$ for elastic net methods correspondingly.

Here we suggest to imitate $f(x)$ with a $PQSQ$ potential function, i.e. instead of (\ref{RegularizedRegression}) solving the problem

\begin{equation}\label{RegularizedRegressionPQSQ}
\frac{1}{N}\sum_{i=1}^N\left(y_i-\sum_{k=1}^m\beta^kx_i^k\right)^2+\lambda\sum_{k=1}^m u(\beta^k) \rightarrow \min,
\end{equation}

\noindent where $u(\beta)$ is a PQSQ potential imitating {\it arbitrary} subquadratic regression regularization penalty.

Solving (\ref{RegularizedRegressionPQSQ}) is equivalent to iteratively solving a system of linear equations

\begin{equation}\label{RegularizedRegressionPQSQIterationEquation}
\begin{split}
&\frac{1}{N}\sum_{k=1}^m\beta^k\sum_{i=1}^N x_i^k x_i^j +\lambda a_{I(\beta^j)}\beta^j \\ & \;\;\;\;\; = \sum_{i=1}^N y_i x_i^j,\, j = 1,\dots , m,
\end{split}
\end{equation}

\noindent where $a_{I(\beta^j)}$ constant (where $I$ index is defined from $r_{I}\leq\beta^j<r_{I+1}$) is computed accordingly to the definition of $u(x)$ function (see (\ref{PQSQ_acoeffs})), given the estimation of $\beta^k$ regression coefficients at the current iteration. In practice, iterating (\ref{RegularizedRegressionPQSQIterationEquation}) converges in a few iterations, therefore, the algorithm can work very fast and outperform the widely used least angle regression algorithm for solving (\ref{RegularizedRegression}) in case of $L_1$ penalties.

\subsection{Introducing sparsity by `black hole' trick}

Any PQSQ potential $u(x)$ is characterized by zero derivative for $x=0$ by construction: $u'(x)|_{x=0}=0$, which means that the solution of (\ref{RegularizedRegressionPQSQ}) does not have to be sparse for any $\lambda$. Unlike pure $L_1$-based penalty, the coefficients of regression diminish with increase of $\lambda$ but there is nothing to shrink them to exact zero values, similar to the ridge regression. However, it is relatively straightforward to modify the algorithm, to achieve sparsity of the regression solution. The 'black hole' trick consists in eliminating from regression training after each iteration (\ref{RegularizedRegressionPQSQIterationEquation}) all regression coefficients $\beta^k$ smaller by absolute value than a given parameter $\epsilon$ ('black hole radius'). Those regression coefficients which have passed the 'black hole radius' are put to zero and do not have any chance to change their values in the next iterations.

\begin{figure}[t]
\centering\includegraphics[width=0.75\linewidth]{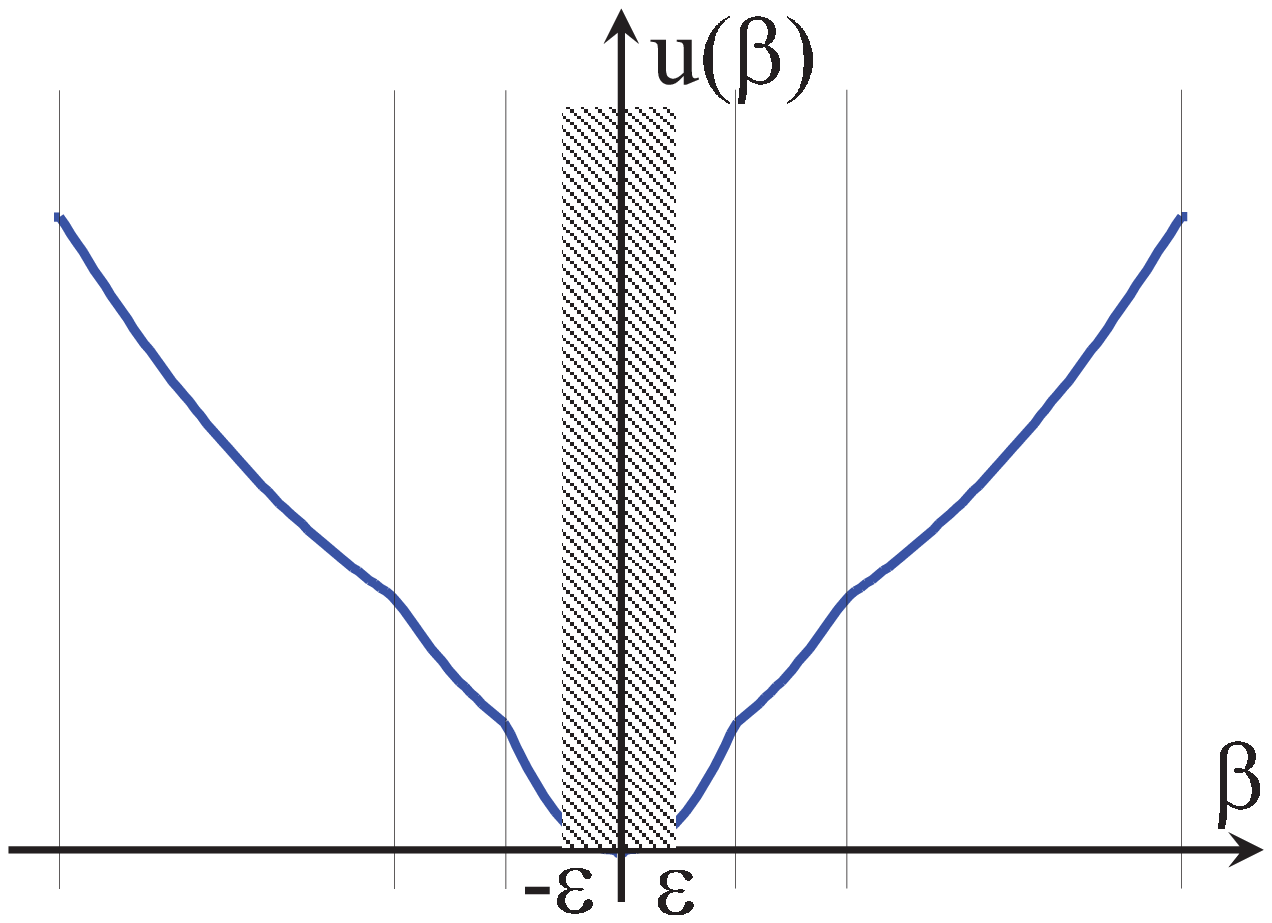}
\caption{'Black hole trick' for introducing sparsity into the PQSQ-based regularized regression. Here PQSQ function imitates $L_1$ norm (for illustration only three intervals are used to define PQSQ function). Black hole trick consists in introducing an $\epsilon$ zone (hatched territory on the plot) of the potential in the vicinity of zero such that any coefficient of regression falling down into this zone is set to zero and eliminated from further learning. It is convenient to define $\epsilon$ as the midst of the smallest interval as it is shown in this plot. \label{ExplainingBlackHole}}
\end{figure}

The optimal choice of $\epsilon$ value requires a separate study. From general considerations, it is preferable that the derivative $u'(x)|_{x=\epsilon}$ would not be very close to zero. As a pragmatic choice for the numerical examples in this article, we define $\epsilon$ as the midst of the smallest interval in the definition of PQSQ potential (see Figure~\ref{ExplainingBlackHole}), i.e. $\epsilon=r_1/2$, which guarantees far from zero $u'(x)|_{x=\epsilon}$. It might happen that this value of $\epsilon$ would collapse all $\beta^k$ to zero even without regularization (i.e., with $\lambda=0$). In this case, the 'black hole radius' is divided by half $\epsilon \leftarrow \epsilon/2$ and it is checked that for $\lambda=0$ the iterations would leave at list half of the regression coefficients. If it is not the case then the process of diminishing the 'black hole radius' repeated recursively till meeting the criterion of preserving the majority of regression coefficients. In practice, it requires only few (very fast) additional iterations of the algorithm.

As in the lasso methodology, the problem (\ref{RegularizedRegressionPQSQ}) is solved for a range of $\lambda$ values, calibrated such that the minimal $\lambda$ would select the maximum number of regression variables, while the maximum $\lambda$ value would lead to the most sparse regression (selecting only one single non-zero regression coefficient).

\section{Numerical examples}

\subsection{Practical choices of parameters}

The main parameters of PQSQ are (a) majorating function $f(x)$ and (b) decomposition of each coordinate range into $p+1$ non-overlapping intervals.
Depending on these parameters, various approximation error properties can be exploited, including those providing robustness to outlier data points.

When defining the intervals $r_j, \, j=1,\dots ,p$, it is desirable to achieve a small difference between $f(\Delta x)$ and $u(\Delta x)$ for expected argument values $\Delta x$ (differences between an estimator and the data point), and choose the suitable value of the potential trimming threshold $r_p$ in order to achieve the desired robustness properties. If no trimming is needed, then $r_p$ should be made larger than the maximum expected difference between coordinate values (maximum $\Delta x$).

In our numerical experiments we used the following definition of intervals. For any data coordinate $k$, we define a characteristic difference $D^k$, for example

\begin{equation}\label{characteristic_distance_amplitude}
D^k = \alpha_{scale}(max_i(x_i^k)-min_i(x_i^k)),
\end{equation}

\noindent where $\alpha_{scale}$ is a scaling parameter, which can be put at 1 (in this case, the approximating potential will not be trimmed). In case of existence of outliers, for defining $D^k$, instead of amplitude one can use other measures such as the median absolute deviation (MAD):

\begin{equation}\label{characteristic_distance_mad}
D^k = \alpha_{scale}median_i(|x_i^k-median(\{x_i^k\})|);
\end{equation}

\noindent in this case, the scaling parameter should be made larger, i.e. $\alpha_{scale}=10$, if no trimming is needed.

After defining $D^k$ we use the following definition of intervals:

\begin{equation}\label{intervals_definition}
r_j^k = D^k\frac{j^2}{p^2}\, , \, j=0, \ldots , p.
\end{equation}

More sophisticated approaches are also possible to apply such as, given the number of intervals $p$ and the majorating function $f(x)$, choose $r_j,\, j=1,\dots ,p$ in order to minimize the maximal difference
$$
d=\max_x|f(x)-u(x)| \rightarrow \min.
$$
The calculation of intervals is straightforward for a given value of $d$ and many smooth concave functions $f(x)$ like $f(x)=x^p$ ($0<p\leq 1$) or $f(x)=\ln(1+x)$.


\subsection{Implementation}

We provide Matlab implementation of PQSQ approximators (in particular, PCA) together with the Matlab and R code used to generate the example
figures in this article at `PQSQ-DataApproximators' GitHub repository\footnote{https://github.com/auranic/PQSQ-DataApproximators} and Matlab implementation of PQSQ-based regularized regression with build-in imitations of $L_1$ (lasso-like) and $L_1\&L_2$ mixture (elastic net-like) penalties at `PQSQ-regularized-regression' GitHub repository\footnote{https://github.com/Mirkes/PQSQ-regularized-regression/}.
The Java code implementing elastic graph-based non-linear approximator implementations is available from
the authors on request.

%

\subsection{Motivating example: dense two-cluster distribution contaminated by sparse noise}

We demonstrate the value of PQSQ-based computation of $L_1$-based PCA by constructing a simple example of data distribution consisting of a dense two-cluster component superimposed with a sparse contaminating component with relatively large variance whose co-variance does not coincide with the dense signal (Figure~\ref{bimodal}). We study the ability of PCA to withstand certain level of sparse contamination and compare it with the standard $L_2$-based PCA. In this example, without noise the first principal component coincides with the vector connecting the two cluster centers: hence, it perfectly separates them in the projected distribution. Noise interferes with the ability of the first principal component to separate the clusters to the degree when the first principal component starts to match the principal variance direction of the contaminating distribution (Figure~\ref{bimodal}A,B). In higher dimensions, not only the first but also the first two principal components are not able to distinguish two clusters, which can hide an important data structure when applying the standard data visualization tools.

In the first test  we study a switch of the first principal component from following the variance of the dense informative distribution (abscissa) to the sparse noise distribution (ordinate) as a function of the number of contaminating points, in $R^2$ (Figure~\ref{bimodal}A-C). We modeled two clusters as two 100-point samples from normal distribution centered in points $[-1;0]$ and $[1;0]$ with isotropic variance with the standard deviation 0.1. The sparse noise distribution was modeled as a $k$-point sample from the product of two Laplace distributions of zero means with the standard deviations 2 along abscissa and 4 along ordinate. The intervals for computing the PQSQ functional were defined by thresholds $R=\{0;0.01;0.1;0.5;1\}$ for each coordinate. Increasing the number of points in the contaminating distribution diminishes the average value of the abscissa coordinate of PC1, because the PC1 starts to be attracted by the contaminating distribution (Figure~\ref{bimodal}C). However, it is clear that on average PQSQ $L_1$-based PCA is able to withstand much larger amplitude of the contaminating signal (very robust up to 20-30 points, i.e. 10-20\% of strong noise contamination) compared to the standard $L_2$-based PCA (which is robust to 2-3\% of contamination).

In the second test we study the ability of the first two principal components to separate two clusters, in $R^{100}$ (Figure~\ref{bimodal}D-F). As in the first test, we modeled two clusters as two 100-point samples from normal distribution centered in points $[-1;0;\dots;0]$ and $[1;0;\dots;0]$ with isotropic variance with the standard deviation 0.1 in all 100 dimensions. The sparse noise distribution is modeled as a $k$-point sample from the product of 100 Laplace distributions of zero means with the standard deviations 1 along each coordinate besides the third coordinate (standard deviation of noise is 2) and the forth coordinate (standard deviation of noise is 4). Therefore, the first two principal component in the absence of noise are attracted by the dimensions 1 and 2, while in the presence of strong noise they are be attracted by dimensions 3 and 4, hiding the cluster structure of the dense part of the distribution. The definitions of the intervals were taken as in the first test. We measured the ability of the first two principal components to separate clusters by computing the $t$-test between the two clusters projected in the 2D-space spanned by the first principal components of the global distribution (Figure~\ref{bimodal}D-F). As one can see, the average ability of the first principal components to separate clusters is significantly stronger in the case of PQSQ $L_1$-based PCA which is able to  separate robustly the clusters even in the presence of strong noise contamination (up to 80 noise points, i.e. 40\% contamination).

\begin{figure*}[ht]
\includegraphics[width=0.95\linewidth]{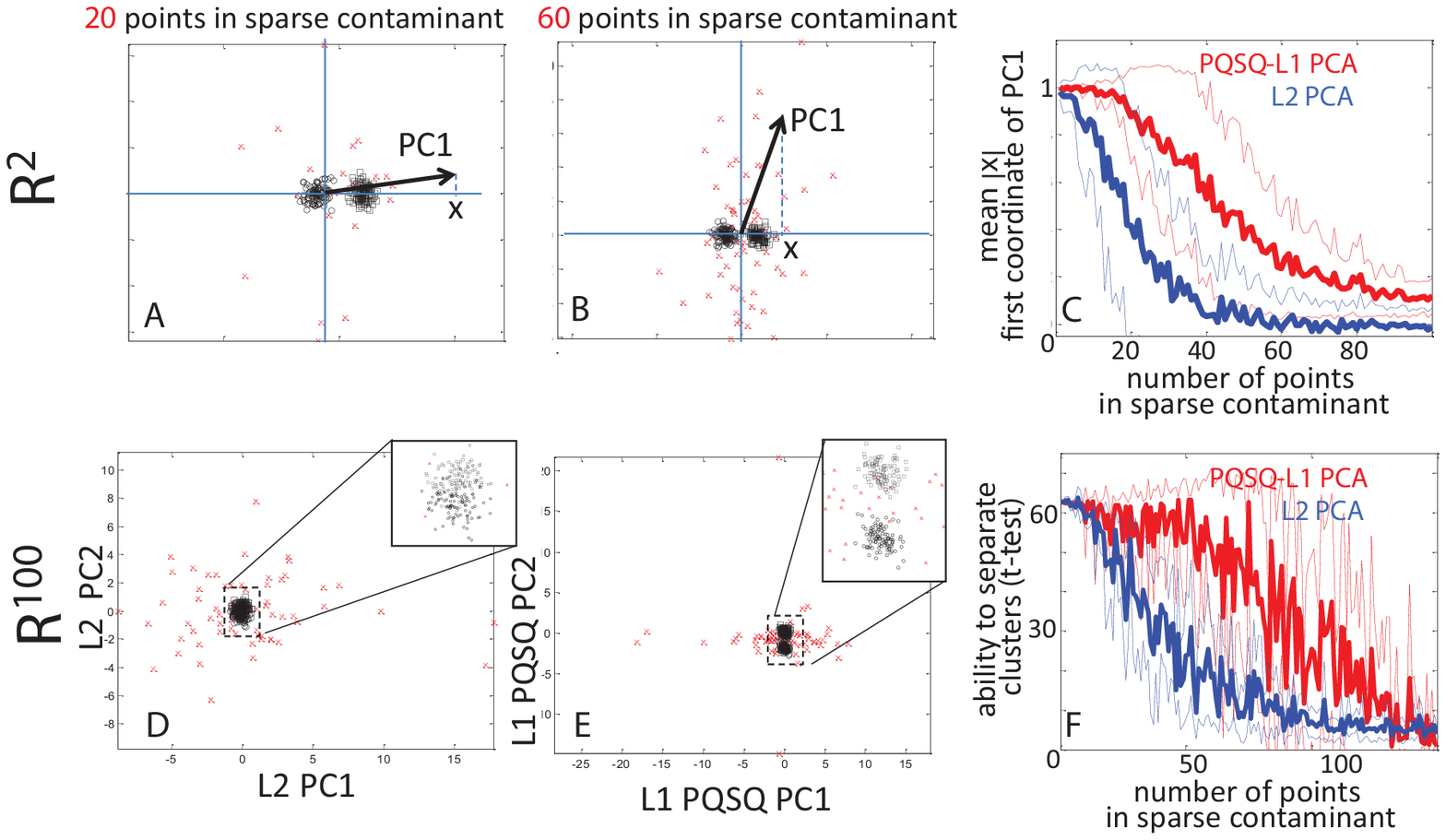}
\caption{Comparing $L_2$- and PQSQ $L_1$-based PCA using example of two-cluster distribution (100 black circles and 100 squares) contaminated by sparse noise (red crosses). A) Dense two cluster distribution contaminated by sparse distribution (20 points) of large variance. In the presence of noise, the abscissa coordinate $x$ of PC1 vector is slightly less than 1. B) Same as A) but in the case of strong contamination (60 points). The value of $x$ is much smaller in this case. C) Average absolute value of the abscissa coordinate of PC1 $|x|$ (thick lines) shown with standard interval (thin lines) for 100 samples of $k$ contaminating points. D) Projection of the data distribution on the first two principal components computed with the standard $L_2$ PCA algorithm. The number of contaminating points is 40. The cluster structure of the dense part of the distribution is completely hidden as shown in the zoom window. E) Same as in D) but computed with PQSQ $L_1$-based algorithm. The cluster structure is perfectly separable. F) The value of $t$-test computed based on the known cluster labels of the dense part of the distribution, in the projections onto the first two principal components of the global distribution. As in C), the mean values of 100 contamination samples together with confidence intervals are shown. \label{bimodal}}
\end{figure*}

\subsection{Performance/stability trade-off benchmarking of $L_1$-based PCA}

In order to compare the computation time and the robustness of PQSQ-based PCA algorithm for the case $u(x)=|x|$ with existing R-based implementations of $L_1$-based PCA methods (pcaL1 package), we followed the benchmark described in \cite{brooks2012pcal1}. We compared performance of PQSQ-based PCA based on {Algorithm \ref{PQSQ_PC1}} with several $L_1$-based PCA algorithms: L1-PCA* \cite{Brooks2013}, L1-PCA \cite{Ke2005}, PCA-PP \cite{Croux2007}, PCA-L1 \cite{Kwak2008}. As a reference point, we used the standard PCA algorithm based on quadratic norm and computed using the standard SVD iterations.

The idea of benchmarking is to generate a series of datasets of the same size ($N=1000$ objects in $m=10$ dimensions) such that the first 5 dimensions would be sampled from a uniform distribution $U(-10,10)$. Therefore, the first 5 dimensions represent `true manifold' sampled by points.

The values in the last 5 dimensions represent `noise+outlier' signal. The background noise is represented by Laplacian distribution of zero mean and 0.1 variance. The outlier signal is characterized by mean value $\mu$, dimension $p$ and frequency $\phi$. Then, for each data point with a probability $\phi$, in the first $p$ outlier dimensions a value is drawn from $Laplace(\mu,0.1)$. The rest of the values is drawn from background noise distribution.

As in \cite{brooks2012pcal1}, we've generated 1300 test sets corresponding to $\phi=0.1$, with 100 examples for each combination of $\mu=1,5,10,25$ and $p=1,2,3$.

For each test set 5 first principal components $\vec{V}_1,\dots,\vec{V}_{5}$ of unit length were computed, with corresponding point projection distributions $U^1,\dots,U^{5}$ and the mean vector $\vec{C}$. Therefore, for each initial data point $\vec{x}_i$, we have the `restored' data point $$P(\vec{x}_i)=\sum_{k=1,\dots ,5}U^k_i\vec{V}_k+\vec{C}.$$

For computing the PQSQ-based PCA we used 5 intervals without trimming. Changing the number of intervals did not significantly changed the benchmarking results.

Two characteristics were measured: (1) computation time measured as a ratio to the computation of 5 first principal components using the standard $L_2$-based PCA and (2) the sum of absolute values of the restored point coordinates in the `outlier' dimensions normalized on the number of points:

\begin{figure}[th]
\centering\includegraphics[width=0.75\linewidth]{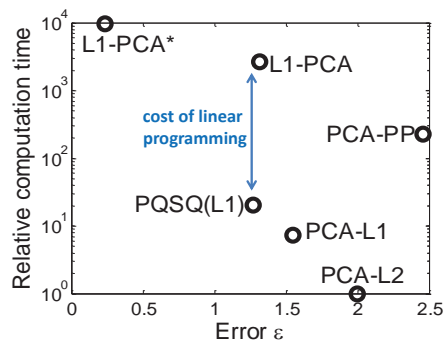}
\caption{Benchmarking several algorithms for constructing $L_1$-based PCA, using synthetic datasets representing `true' five-dimensional linear manifold contaminated by noise and outliers (located in other five dimensions). The abscissa is the error of detecti tnghe `true' manifold by a particular method and the ordinate is the computational time relative to the standard SVD ($L_2$-based PCA) computation, in logarithmic scale. The computational cost of application of linear programming methods instead of simpler iterative methods is approximately shown by an arrow. \label{benchmark}}
\end{figure}

\begin{equation}
\sigma = \frac{1}{N}\sum_{i=1,\dots , N} \sum_{k=6,\dots ,10} |P^k(\vec{x}_i)|.
\end{equation}

Formally speaking, $\sigma$ is $L_1$-based distance from the point projection onto the first five principal components to the `true' subspace. In simplistic terms, larger values of $\sigma$ correspond to the situation when the first five principal components do not represent well the first `true' dimensions having significant loadings into the `outlier dimensions'.  $\sigma=0$ if and only if the first five components do not have any non-zero loadings in the dimensions $6, \dots ,10$.

The results of the comparison averaged over all 1300 test sets are shown in Figure~\ref{benchmark}. The PQSQ-based computation of PCA outperforms by accuracy the existing heuristics such as PCA-L1 but remains computationally efficient. It is 100 times faster than L1-PCA giving almost the same accuracy. It is almost 500 times faster than the L1-PCA* algorithm, which is, however, better in accuracy (due to being robust with respect to strong outliers).  One can see from Figure~\ref{benchmark} that PQSQ-based approach is the best in accuracy
among fast iterative methods. The detailed tables of comparison for all combinations of parameters are available on GitHub\footnote{http://goo.gl/sXBvqh}. The scripts used to generate the datasets and compare the results can also be found there\footnote{https://github.com/auranic/PQSQ-DataApproximators}.

\subsection{Comparing performances of PQSQ-based regularized regression and lasso algorithms}

We compared performance of PQSQ-based regularized regression method imitating $L_1$ penalty with lasso implementation in Matlab, using 8 datasets from UCI
Machine Learning Repository \cite{Lichman:2013}, Regression Task section. In the selection of datasets we chose very different numbers of objects and variables for regression construction (Table~\ref{lassoPerfTable}). All table rows containing missing values were eliminated for the tests.
\begin{table*}[]
\centering{
\caption{Comparing time performance (in seconds, on ordinary laptop) of lasso vs PQSQ-based regularized regression imitating $L_1$ penalty. Average acceleration of PQSQ-based method vs lasso in these 8 examples is 120 fold with comparable accuracy.
\label{lassoPerfTable}}
\begin{tabular}{llllll}
\hline
{Dataset}               & {Objects} & {Variables} & {lasso} & {PQSQ} & {Ratio} \\
\hline Breast cancer                  & 47               & 31                 & 10.50          & 0.05          & 233.33         \\
Prostate cancer                & 97               & 8                  & 0.07           & 0.02          & 4.19           \\
ENB2012                        & 768              & 8                  & 0.53           & 0.03          & 19.63          \\
Parkinson                      & 5875             & 26                 & 20.30          & 0.04          & 548.65         \\
Crime                          & 1994             & 100                & 10.50          & 0.19          & 56.24          \\
Crime reduced                  & 200              & 100                & 17.50          & 0.17          & 102.94          \\
Forest fires                   & 517              & 8                  & 0.05           & 0.02          & 3.06           \\
Random regression (1000$\times$250) & 1000             & 250                & 2.82           & 0.58          & 4.86
 \\ \hline
\end{tabular}}
\end{table*}

We observed up to two orders of magnitude acceleration of PQSQ-based method compared to the lasso method implemented in Matlab (Table~\ref{lassoPerfTable}),
with similar sparsity properties and approximation power as lasso (Figure~\ref{lassoComparisonMSE}).

While comparing time performances of two methods, we've noticed that lasso (as it is implemented in Matlab) showed worse results when the number of objects in the dataset approaches the number of predictive variables (see Table~\ref{lassoPerfTable}). In order to test this observation explicitly, we took a particular dataset ('Crime') containing 1994 observations and 100 variables and compared the performance of lasso in the case of complete table and a reduced table ('Crime reduced') containing only each 10th observation. Paradoxically, lasso converges almost two times slower in the case of the smaller dataset, while the PQSQ-based algorithm worked slightly faster in this case. 

It is necessary to stress that here we compare the basic algorithms without many latest technical improvements which can be applied both to $L_1$ penalty and its PQSQ approximation (such as fitting the whole lasso path). Detailed comparison of all the existent modifications if far beyond the scope of this work.

For comparing approximation power of the PQSQ-based regularized regression and lasso, we used two versions of PQSQ potential for regression coefficients: with and without trimming. In order to represent the results, we used the `Number of non-zero parameters vs Fraction of Variance Unexplained (FVU)' plots (see two representative examples at Figure~\ref{lassoComparisonMSE}). We suggest that this type of plot is more informative in practical applications than the 'lasso plot' used to calibrate the strength of regularization, since it is a more explicit representation for optimizing the accuracy vs complexity ratio of the resulting regression.

From our testing, we can conclude that PQSQ-based regularized regression has similar properties of sparsity and approximation accuracy compared to lasso. It tends to slightly outperform lasso (to give smaller FVU) in case of $N\approx P$. Introducing trimming in most cases does not change the best FVU for a given number of selected variables, but tends to decrease its variance (has a stabilization effect). In some cases, introducing trimming is the most advantageous method (Figure~\ref{lassoComparisonMSE}B).

The GitHub `PQSQ-regularized-regression' repository contains exact dataset references and more complete report on comparing approximation ability of PQSQ-based regularized regression with lasso\footnote{https://github.com/Mirkes/PQSQ-regularized-regression/wiki}.

\begin{figure*}[th]
\centering\includegraphics[width=0.95\linewidth]{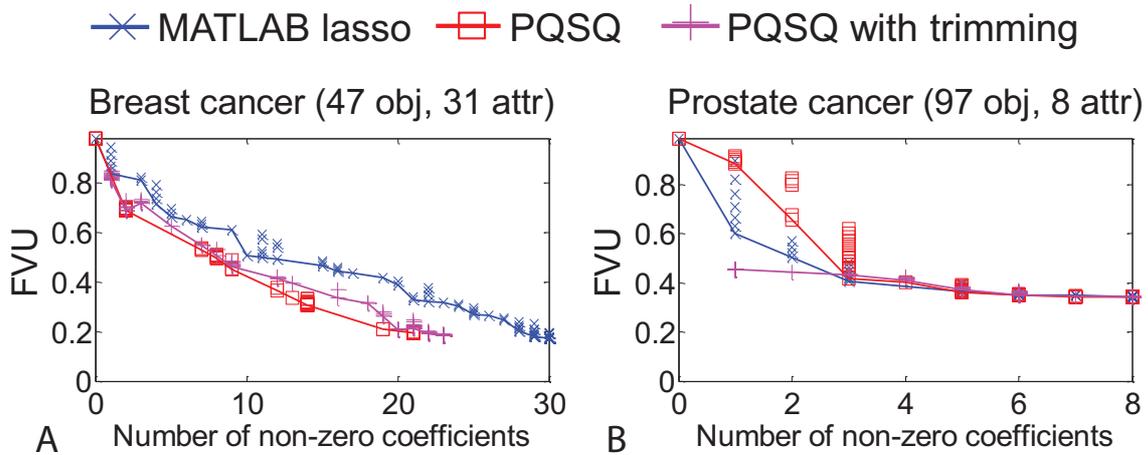}
\caption{Number of non-zero regression coefficients vs FVU plot for two example real-life datasets (A - Breast cancer Wisconsin dataset from UC Irvine Machine Learning Repository, B - original prostate cancer example from the seminal lasso paper\cite{Tibshirani1996}). Each cross shows a particular solution of the regularized regression problem. Solid lines show the best (minimal) obtained FVU within the same number of selected variables. \label{lassoComparisonMSE}}
\end{figure*}

\section{Conclusion}

In this paper we develop a new machine learning framework (theory and application) allowing one to deal with arbitrary error potentials of not-faster than quadratic growth, imitated by piece-wise quadratic function of subquadratic growth (PQSQ error potential).

We develop methods for constructing the standard data approximators (mean value, $k$-means clustering, principal components, principal graphs)
for arbitrary non-quadratic approximation error with subquadratic growth and regularized linear regression with arbitrary subquadratic penalty by using a piecewise-quadratic error functional (PQSQ potential). These problems can be solved by applying quasi-quadratic optimization procedures, which are organized as solutions of sequences of linear problems by standard and computationally efficient algorithms.

\begin{table*}\centering{
\caption{List of methods which can use PQSQ-based error potentials \label{ApplicationTable}}
\begin{tabular*}{\textwidth}{p{6cm}p{7cm}}
\hline
\multicolumn{2}{c}{Data approximation/Clustering/Manifold learning} \\
\hline
Principal Component Analysis &  Includes robust trimmed version of PCA, $L_1$-based PCA, regularized PCA, and many other PCA modifications \\
Principal curves and manifolds &  Provides possibility to use non-quadratic data approximation terms and trimmed robust version\\
Self-Organizing maps &  Same as above \\
Principal graphs/trees &  Same as above \\
K-means & Can include adaptive error potentials based on estimated error distributions inside clusters \\
\hline
\multicolumn{2}{c}{High-dimensional data mining} \\
\hline
Use of fractional quasinorms $L_p$~($0<p<1$) & Introducing fractional quasinorms in existing data-mining techniques can potentially deal with the curse of dimensionality, helping to better distinguish close from distant data points \cite{Aggarwal2001} \\
\hline
\multicolumn{2}{c}{Regularized and sparse regression} \\
\hline
Lasso & Application of PQSQ-based potentials leads to speeding up in case of large and $N\approx P$ datasets \\
Elastic net & Same as above \\
\hline
\end{tabular*}}
\end{table*}

The suggested methodology have several advantages over existing ones:

(a) \textit{Scalability}: the algorithms are computationally efficient and can be applied to large data sets containing millions of numerical values.

(b) \textit{Flexibility}: the algorithms can be adapted to any type of data metrics with subquadratic growth, even if the metrics can not be expressed in explicit form. For example, the error potential can be chosen as adaptive metrics \cite{Yang2006, Wu2009}.

(c) \textit{Built-in (trimmed) robustness}: choice of intervals in PQSQ can be done in the way to achieve a trimmed version of the standard data approximators, when points distant from the approximator do not affect the error minimization during the current optimization step.

(d) \textit{Guaranteed convergence}: the suggested algorithms converge to local or global minimum just as the corresponding predecessor algorithms based on quadratic optimization and expectation/minimization-based splitting approach.

In theoretical perspective, using PQSQ-potentials in data mining is similar to existing applications of min-plus (or, max-plus) algebras in non-linear optimization theory, where complex non-linear functions are approximated by infimum (or supremum) of finitely many `dictionary functions' \cite{Gaubert2011,Magron2015}. We can claim that just as using polynomes is a natural framework for approximating in rings of  functions, using min-plus algebra naturally leads to introduction of PQSQ-based functions and the  cones of minorants of quadratic dictionary functions.

One of the application of the suggested methodology is approximating the popular in data mining $L_1$ metrics. We show by numerical simulations that PQSQ-based approximators perform as fast as the fast heuristical algorithms for computing $L_1$-based PCA but achieve better accuracy in a previously suggested benchmark test. PQSQ-based approximators can be less accurate than the exact algorithms for optimizing $L_1$-based functions utilizing linear programming: however, they are several orders of magnitude faster. PQSQ potential can be applied in the task of regression, replacing the classical Least Squares or $L_1$-based Least Absolute Deviation methods. At the same time, PQSQ-based approximators can imitate a variety of subquadratic error potentials (not limited to $L_1$ or variations), including fractional quasinorms $L_p$ ($0<p<1$). We demonstrate that the PQSQ potential can be easily adapted to the problems of sparse regularized regression with non-quadratic penalty on regression coefficients (including imitations of lasso and elastic net). On several real-life dataset examples we show that PQSQ-based regularized regression can perform two orders of magnitude faster than the lasso algorithm implemented in the same programming environment. 

To conclude, in Table~\ref{ApplicationTable} we list possible applications of the PQSQ-based framework in machine learning. 

\section*{Aknowledgement}
This study was supported in part by Big Data Paris Science et Lettre Research University project `PSL Institute for Data Science'.


\begin{thebibliography}{39}
\expandafter\ifx\csname natexlab\endcsname\relax\def\natexlab#1{#1}\fi
\providecommand{\bibinfo}[2]{#2}
\ifx\xfnm\relax \def\xfnm[#1]{\unskip,\space#1}\fi
\bibitem[{Tibshirani(1996)}]{Tibshirani1996}
\bibinfo{author}{R.~Tibshirani},
\newblock \bibinfo{title}{Regression shrinkage and selection via the lasso},
\newblock \bibinfo{journal}{Journal of the Royal Statistical Society. Series B
  (Methodological)}  (\bibinfo{year}{1996}) \bibinfo{pages}{267--288}.
\bibitem[{Zou and Hastie(2005)}]{Zou2005}
\bibinfo{author}{H.~Zou}, \bibinfo{author}{T.~Hastie},
\newblock \bibinfo{title}{Regularization and variable selection via the elastic
  net},
\newblock \bibinfo{journal}{Journal of the Royal Statistical Society: Series B
  (Statistical Methodology)} \bibinfo{volume}{67} (\bibinfo{year}{2005})
  \bibinfo{pages}{301--320}.
\bibitem[{Barillot et~al.(2012)Barillot, Calzone, Hupe, Vert, and
  Zinovyev}]{Barillot2012}
\bibinfo{author}{E.~Barillot}, \bibinfo{author}{L.~Calzone},
  \bibinfo{author}{P.~Hupe}, \bibinfo{author}{J.-P. Vert},
  \bibinfo{author}{A.~Zinovyev}, \bibinfo{title}{Computational Systems Biology
  of Cancer}, \bibinfo{publisher}{Chapman \& Hall, CRC Mathemtical and
  Computational Biology}, \bibinfo{year}{2012}.
\bibitem[{Wright et~al.(2010)Wright, Ma, Mairal, Sapiro, Huang, and
  Yan}]{Wright2010}
\bibinfo{author}{J.~Wright}, \bibinfo{author}{Y.~Ma},
  \bibinfo{author}{J.~Mairal}, \bibinfo{author}{G.~Sapiro},
  \bibinfo{author}{T.~S. Huang}, \bibinfo{author}{S.~Yan},
\newblock \bibinfo{title}{Sparse representation for computer vision and pattern
  recognition},
\newblock \bibinfo{journal}{Proceedings of the IEEE} \bibinfo{volume}{98}
  (\bibinfo{year}{2010}) \bibinfo{pages}{1031--1044}.
\bibitem[{Lloyd(1957)}]{Lloyd1957}
\bibinfo{author}{S.~Lloyd},
\newblock \bibinfo{title}{Last square quantization in pcm's},
\newblock \bibinfo{journal}{Bell Telephone Laboratories Paper}
  (\bibinfo{year}{1957}).
\bibitem[{Pearson(1901)}]{Pearson1901On}
\bibinfo{author}{K.~Pearson},
\newblock \bibinfo{title}{{On lines and planes of closest fit to systems of
  points in space}},
\newblock \bibinfo{journal}{Philos. Mag.} \bibinfo{volume}{2}
  (\bibinfo{year}{1901}) \bibinfo{pages}{559--572}.
\bibitem[{Gorban and Zinovyev(2009)}]{Gorban2009}
\bibinfo{author}{A.~N. Gorban}, \bibinfo{author}{A.~Zinovyev},
\newblock \bibinfo{title}{Principal graphs and manifolds},
\newblock \bibinfo{journal}{In Handbook of Research on Machine Learning
  Applications and Trends: Algorithms, Methods and Techniques, eds. Olivas
  E.S., Guererro J.D.M., Sober M.M., Benedito J.R.M., Lopes A.J.S.}
  (\bibinfo{year}{2009}).
\bibitem[{Gorban et~al.(2008)Gorban, Kegl, Wunsch, and
  Zinovyev}]{Gorban2008Principal}
\bibinfo{editor}{A.~Gorban}, \bibinfo{editor}{B.~Kegl},
  \bibinfo{editor}{D.~Wunsch}, \bibinfo{editor}{A.~Zinovyev} (Eds.),
  \bibinfo{title}{Principal Manifolds for Data Visualisation and Dimension
  Reduction, LNCSE 58}, \bibinfo{publisher}{Springer}, \bibinfo{year}{2008}.
\bibitem[{Flury(1990)}]{Flury1990}
\bibinfo{author}{B.~Flury},
\newblock \bibinfo{title}{Principal points},
\newblock \bibinfo{journal}{Biometrika} \bibinfo{volume}{77}
  (\bibinfo{year}{1990}) \bibinfo{pages}{33--41}.
\bibitem[{Hastie(1984)}]{Hastie1984}
\bibinfo{author}{T.~Hastie},
\newblock \bibinfo{title}{Principal curves and surfaces},
\newblock \bibinfo{journal}{PhD Thesis, Stanford University, California}
  (\bibinfo{year}{1984}).
\bibitem[{Gorban et~al.(2007)Gorban, Sumner, and
  Zinovyev}]{gorban2007topological}
\bibinfo{author}{A.~N. Gorban}, \bibinfo{author}{N.~R. Sumner},
  \bibinfo{author}{A.~Y. Zinovyev},
\newblock \bibinfo{title}{Topological grammars for data approximation},
\newblock \bibinfo{journal}{Applied Mathematics Letters} \bibinfo{volume}{20}
  (\bibinfo{year}{2007}) \bibinfo{pages}{382--386}.
\bibitem[{Aggarwal et~al.(2001)Aggarwal, Hinneburg, and Keim}]{Aggarwal2001}
\bibinfo{author}{C.~C. Aggarwal}, \bibinfo{author}{A.~Hinneburg},
  \bibinfo{author}{D.~A. Keim},
\newblock \bibinfo{title}{On the surprising behavior of distance metrics in
  high dimensional space},
\newblock in: \bibinfo{booktitle}{Database Theory - ICDT 2001, 8th
  International Conference, London, UK, January 4-6, 2001, Proceedings},
  \bibinfo{publisher}{Springer}, \bibinfo{year}{2001}, pp.
  \bibinfo{pages}{420--434}.
\bibitem[{Ding et~al.(2006)Ding, Zhou, He, and Zha}]{Ding2006}
\bibinfo{author}{C.~Ding}, \bibinfo{author}{D.~Zhou}, \bibinfo{author}{X.~He},
  \bibinfo{author}{H.~Zha},
\newblock \bibinfo{title}{R1-{PCA}: rotational invariant {L}1-norm principal
  component analysis for robust subspace factorization},
\newblock \bibinfo{journal}{ICML}  (\bibinfo{year}{2006})
  \bibinfo{pages}{281--288}.
\bibitem[{Hauberg et~al.(2014)Hauberg, Feragen, and Black}]{hauberg2014}
\bibinfo{author}{S.~Hauberg}, \bibinfo{author}{A.~Feragen},
  \bibinfo{author}{M.~J. Black},
\newblock \bibinfo{title}{Grassmann averages for scalable robust pca},
\newblock in: \bibinfo{booktitle}{Computer Vision and Pattern Recognition
  (CVPR), 2014 IEEE Conference on}, \bibinfo{organization}{IEEE}, pp.
  \bibinfo{pages}{3810--3817}.
\bibitem[{Xu and Yuille(1995)}]{Xu1995}
\bibinfo{author}{L.~Xu}, \bibinfo{author}{A.~L. Yuille},
\newblock \bibinfo{title}{Robust principal component analysis by
  self-organizing rules based on statistical physics approach},
\newblock \bibinfo{journal}{Neural Networks, IEEE Transactions on}
  \bibinfo{volume}{6} (\bibinfo{year}{1995}) \bibinfo{pages}{131--143}.
\bibitem[{Allende et~al.(2004)Allende, Rogel, Moreno, and Salas}]{Allende2004}
\bibinfo{author}{H.~Allende}, \bibinfo{author}{C.~Rogel},
  \bibinfo{author}{S.~Moreno}, \bibinfo{author}{R.~Salas},
\newblock \bibinfo{title}{Robust neural gas for the analysis of data with
  outliers},
\newblock in: \bibinfo{booktitle}{Computer Science Society, 2004. SCCC 2004.
  24th International Conference of the Chilean}, \bibinfo{organization}{IEEE},
  pp. \bibinfo{pages}{149--155}.
\bibitem[{Kohonen(2001)}]{kohonen2001self}
\bibinfo{author}{T.~Kohonen}, \bibinfo{title}{Self-organizing Maps}, Springer
  Series in Information Sciences, Vol.30, \bibinfo{publisher}{Berlin,
  Springer}, \bibinfo{year}{2001}.
\bibitem[{Jolliffe et~al.(2003)Jolliffe, Trendafilov, and Uddin}]{Jolliffe2003}
\bibinfo{author}{I.~T. Jolliffe}, \bibinfo{author}{N.~T. Trendafilov},
  \bibinfo{author}{M.~Uddin},
\newblock \bibinfo{title}{A modified principal component technique based on the
  lasso},
\newblock \bibinfo{journal}{Journal of computational and Graphical Statistics}
  \bibinfo{volume}{12} (\bibinfo{year}{2003}) \bibinfo{pages}{531--547}.
\bibitem[{Cand{\`e}s et~al.(2011)Cand{\`e}s, Li, Ma, and Wright}]{Candes2011}
\bibinfo{author}{E.~J. Cand{\`e}s}, \bibinfo{author}{X.~Li},
  \bibinfo{author}{Y.~Ma}, \bibinfo{author}{J.~Wright},
\newblock \bibinfo{title}{Robust principal component analysis?},
\newblock \bibinfo{journal}{Journal of the ACM (JACM)} \bibinfo{volume}{58}
  (\bibinfo{year}{2011}) \bibinfo{pages}{11}.
\bibitem[{Zou et~al.(2006)Zou, Hastie, and Tibshirani}]{Zou2006}
\bibinfo{author}{H.~Zou}, \bibinfo{author}{T.~Hastie},
  \bibinfo{author}{R.~Tibshirani},
\newblock \bibinfo{title}{Sparse principal component analysis},
\newblock \bibinfo{journal}{Journal of computational and graphical statistics}
  \bibinfo{volume}{15} (\bibinfo{year}{2006}) \bibinfo{pages}{265--286}.
\bibitem[{Cuesta-Albertos et~al.(1997)Cuesta-Albertos, Gordaliza, Matr{\'a}n
  et~al.}]{cuesta1997}
\bibinfo{author}{J.~Cuesta-Albertos}, \bibinfo{author}{A.~Gordaliza},
  \bibinfo{author}{C.~Matr{\'a}n}, et~al.,
\newblock \bibinfo{title}{Trimmed $ k $-means: An attempt to robustify
  quantizers},
\newblock \bibinfo{journal}{The Annals of Statistics} \bibinfo{volume}{25}
  (\bibinfo{year}{1997}) \bibinfo{pages}{553--576}.
\bibitem[{Lu et~al.(2015)Lu, Lin, and Yan}]{Lu2015}
\bibinfo{author}{C.~Lu}, \bibinfo{author}{Z.~Lin}, \bibinfo{author}{S.~Yan},
\newblock \bibinfo{title}{Smoothed low rank and sparse matrix recovery by
  iteratively reweighted least squares minimization.},
\newblock \bibinfo{journal}{IEEE Trans Image Process} \bibinfo{volume}{24}
  (\bibinfo{year}{2015}) \bibinfo{pages}{646--654}.
\bibitem[{Ke and Kanade(2005)}]{Ke2005}
\bibinfo{author}{Q.~Ke}, \bibinfo{author}{T.~Kanade},
\newblock \bibinfo{title}{Robust l 1 norm factorization in the presence of
  outliers and missing data by alternative convex programming},
\newblock in: \bibinfo{booktitle}{Computer Vision and Pattern Recognition,
  2005. CVPR 2005. IEEE Computer Society Conference on},
  volume~\bibinfo{volume}{1}, \bibinfo{organization}{IEEE}, pp.
  \bibinfo{pages}{739--746}.
\bibitem[{Kwak(2008)}]{Kwak2008}
\bibinfo{author}{N.~Kwak},
\newblock \bibinfo{title}{Principal component analysis based on {L1}-norm
  maximization},
\newblock \bibinfo{journal}{Pattern Analysis and Machine Intelligence, IEEE
  Transactions on} \bibinfo{volume}{30} (\bibinfo{year}{2008})
  \bibinfo{pages}{1672--1680}.
\bibitem[{Brooks et~al.(2013)Brooks, Dul{\'{a}}, and Boone}]{Brooks2013}
\bibinfo{author}{J.~Brooks}, \bibinfo{author}{J.~Dul{\'{a}}},
  \bibinfo{author}{E.~Boone},
\newblock \bibinfo{title}{A pure {L1}-norm principal component analysis},
\newblock \bibinfo{journal}{Comput Stat Data Anal} \bibinfo{volume}{61}
  (\bibinfo{year}{2013}) \bibinfo{pages}{83--98}.
\bibitem[{Brooks and Jot(2012)}]{brooks2012pcal1}
\bibinfo{author}{J.~Brooks}, \bibinfo{author}{S.~Jot}, \bibinfo{title}{{PCAL}1:
  An implementation in {R} of three methods for {L1}-norm principal component
  analysis, {O}ptimization {O}nline preprint},
  \bibinfo{howpublished}{\url{http://www.optimization-online.org/DB_HTML/2012/04/3436.html}},
  \bibinfo{year}{2012}.
\bibitem[{Park and Klabjan(2014)}]{Park2014}
\bibinfo{author}{Y.~W. Park}, \bibinfo{author}{D.~Klabjan},
  \bibinfo{title}{Algorithms for {L1}-norm principal component analysis.
  {T}utorial},
  \bibinfo{howpublished}{\url{http://dynresmanagement.com/uploads/3/3/2/9/3329212/algorithms_for_l1pca.pdf}},
  \bibinfo{year}{2014}.
\bibitem[{Gorban' and Rossiev(1999)}]{Gorban1999}
\bibinfo{author}{A.~Gorban'}, \bibinfo{author}{A.~Rossiev},
\newblock \bibinfo{title}{Neural network iterative method of principal curves
  for data with gaps},
\newblock \bibinfo{journal}{Journal of Computer and Systems Sciences
  International c/c of Tekhnicheskaia Kibernetika} \bibinfo{volume}{38}
  (\bibinfo{year}{1999}) \bibinfo{pages}{825--830}.
\bibitem[{Gorban and Zinovyev(2001{\natexlab{a}})}]{Gorban2001ihespreprint}
\bibinfo{author}{A.~Gorban}, \bibinfo{author}{A.~Zinovyev},
\newblock \bibinfo{title}{Visualization of data by method of elastic maps and
  its applications in genomics, economics and sociology},
\newblock \bibinfo{journal}{IHES Preprints}
  (\bibinfo{year}{2001}{\natexlab{a}}).
\bibitem[{Gorban and Zinovyev(2001{\natexlab{b}})}]{gorban2001method}
\bibinfo{author}{A.~Gorban}, \bibinfo{author}{A.~Y. Zinovyev},
\newblock \bibinfo{title}{Method of elastic maps and its applications in data
  visualization and data modeling},
\newblock \bibinfo{journal}{International Journal of Computing Anticipatory
  Systems, CHAOS} \bibinfo{volume}{12} (\bibinfo{year}{2001}{\natexlab{b}})
  \bibinfo{pages}{353--369}.
\bibitem[{Gorban and Zinovyev(2005)}]{gorban2005elastic}
\bibinfo{author}{A.~Gorban}, \bibinfo{author}{A.~Zinovyev},
\newblock \bibinfo{title}{Elastic principal graphs and manifolds and their
  practical applications},
\newblock \bibinfo{journal}{Computing} \bibinfo{volume}{75}
  (\bibinfo{year}{2005}) \bibinfo{pages}{359--379}.
\bibitem[{Gorban and Zinovyev(2010)}]{Gorban2010}
\bibinfo{author}{A.~N. Gorban}, \bibinfo{author}{A.~Zinovyev},
\newblock \bibinfo{title}{Principal manifolds and graphs in practice: from
  molecular biology to dynamical systems.},
\newblock \bibinfo{journal}{Int J Neural Syst} \bibinfo{volume}{20}
  (\bibinfo{year}{2010}) \bibinfo{pages}{219--232}.
\bibitem[{Gorban et~al.(2014)Gorban, Pitenko, and Zinovyev}]{Gorban2014}
\bibinfo{author}{A.~N. Gorban}, \bibinfo{author}{A.~Pitenko},
  \bibinfo{author}{A.~Zinovyev},
\newblock \bibinfo{title}{Vi{D}a{E}xpert: user-friendly tool for nonlinear
  visualization and analysis of multidimensional vectorial data},
\newblock \bibinfo{journal}{arXiv preprint arXiv:1406.5550}
  (\bibinfo{year}{2014}).
\bibitem[{Croux et~al.(2007)Croux, Filzmoser, and Oliveira}]{Croux2007}
\bibinfo{author}{C.~Croux}, \bibinfo{author}{P.~Filzmoser},
  \bibinfo{author}{M.~R. Oliveira},
\newblock \bibinfo{title}{Algorithms for projection--pursuit robust principal
  component analysis},
\newblock \bibinfo{journal}{Chemometrics and Intelligent Laboratory Systems}
  \bibinfo{volume}{87} (\bibinfo{year}{2007}) \bibinfo{pages}{218--225}.
\bibitem[{Lichman(2013)}]{Lichman:2013}
\bibinfo{author}{M.~Lichman}, \bibinfo{title}{University of {C}alifornia,
  {I}rvine ({UCI}) {M}achine {L}earning {R}epository},
  \bibinfo{howpublished}{\url{http://archive.ics.uci.edu/ml}},
  \bibinfo{year}{2013}.
\bibitem[{Yang and Jin(2006)}]{Yang2006}
\bibinfo{author}{L.~Yang}, \bibinfo{author}{R.~Jin},
\newblock \bibinfo{title}{Distance metric learning: A comprehensive survey},
\newblock \bibinfo{journal}{Michigan State Universiy} \bibinfo{volume}{2}
  (\bibinfo{year}{2006}).
\bibitem[{Wu et~al.(2009)Wu, Jin, Hoi, Zhu, and Yu}]{Wu2009}
\bibinfo{author}{L.~Wu}, \bibinfo{author}{R.~Jin}, \bibinfo{author}{S.~C. Hoi},
  \bibinfo{author}{J.~Zhu}, \bibinfo{author}{N.~Yu},
\newblock \bibinfo{title}{Learning {B}regman distance functions and its
  application for semi-supervised clustering},
\newblock in: \bibinfo{booktitle}{Advances in neural information processing
  systems}, pp. \bibinfo{pages}{2089--2097}.
\bibitem[{Gaubert et~al.(2011)Gaubert, McEneaney, and Qu}]{Gaubert2011}
\bibinfo{author}{S.~Gaubert}, \bibinfo{author}{W.~McEneaney},
  \bibinfo{author}{Z.~Qu},
\newblock \bibinfo{title}{Curse of dimensionality reduction in max-plus based
  approximation methods: theoretical estimates and improved pruning
  algorithms},
\newblock \bibinfo{journal}{Arxiv preprint} \bibinfo{volume}{1109.5241}
  (\bibinfo{year}{2011}).
\bibitem[{Magron et~al.(2015)Magron, Allamigeon, Gaubert, and
  Werner}]{Magron2015}
\bibinfo{author}{V.~Magron}, \bibinfo{author}{X.~Allamigeon},
  \bibinfo{author}{S.~Gaubert}, \bibinfo{author}{B.~Werner},
\newblock \bibinfo{title}{Formal proofs for nonlinear optimization},
\newblock \bibinfo{journal}{Arxiv preprint} \bibinfo{volume}{1404.7282}
  (\bibinfo{year}{2015}).

\end{thebibliography}
\end{document}